\documentclass{article}

\usepackage{microtype}
\usepackage{graphicx}
\usepackage{subfigure}
\usepackage{booktabs} 

\usepackage{amsthm}
\usepackage{amsmath}
\usepackage{amssymb}
\usepackage{enumerate}

\usepackage[noend]{algorithmic}

\usepackage{hyperref}

\usepackage{xcolor}
\usepackage{floatrow}
\newfloatcommand{capbtabbox}{table}[][\FBwidth]
\usepackage[draft]{fixme}
\fxsetup{theme=color}
\definecolor{fxnote}{HTML}{268bd2}
\definecolor{fxtarget}{RGB}{220,50,47}

\usepackage{makecell}

\newtheorem{theorem}{Theorem}[section]

\theoremstyle{definition}

\newtheorem{remark}[theorem]{Remark}
\newtheorem{example}[theorem]{Example}

\def\RSet{\mathbb{R}}
\newcommand{\N}{\mathbb{N}}
\newcommand{\R}{\mathbb{R}}
\usepackage{xcolor}
\DeclareMathOperator{\argmin}{argmin}

\usepackage[accepted]{icml2018}

\icmltitlerunning{Bilevel Programming for Hyperparameter Optimization and Meta-Learning}

\begin{document}

\twocolumn[
\icmltitle{
Bilevel Programming for Hyperparameter Optimization and Meta-Learning}
%



\icmlsetsymbol{equal}{*}

\begin{icmlauthorlist}
\icmlauthor{Luca Franceschi}{iit,ucl}
\icmlauthor{Paolo Frasconi}{ufi}
\icmlauthor{Saverio Salzo}{iit}
\icmlauthor{Riccardo Grazzi}{iit}
\icmlauthor{Massimiliano Pontil}{iit,ucl}
\end{icmlauthorlist}

\icmlaffiliation{iit}{Computational Statistics and Machine Learning, Istituto Italiano di Tecnologia, Genoa, Italy}
\icmlaffiliation{ucl}{Department of Computer Science, University College London, London, UK}
\icmlaffiliation{ufi}{Department of Information Engineering, Universit{\`a} degli Studi di Firenze, Florence, Italy}

\icmlcorrespondingauthor{Luca Franceschi}{luca.franceschi@iit.it}

\icmlkeywords{Machine Learning, ICML}

\vskip 0.3in
]



\printAffiliationsAndNotice{}  

\begin{abstract}
  We introduce a framework based on bilevel programming that unifies gradient-based hyperparameter optimization and meta-learning. We show that an approximate version of the bilevel problem can be solved by taking into explicit account the optimization dynamics for the inner objective.  Depending on the specific setting, the outer variables take either the meaning of hyperparameters in a supervised learning problem or parameters of a meta-learner.
We provide sufficient conditions under which solutions of the approximate problem converge to those of the exact problem. We instantiate our approach for meta-learning in the case of  deep learning where representation layers are treated as hyperparameters shared across a set of training episodes. 
In experiments, we confirm our theoretical findings, present encouraging results for few-shot learning and contrast the bilevel approach against classical approaches for learning-to-learn.
\end{abstract}

\section{Introduction} \label{sec:intro}

While in standard supervised learning problems we seek the best hypothesis in a given space and with a given learning algorithm, in hyperparameter optimization (HO) and meta-learning (ML)
we seek a configuration so that the 
optimized
learning algorithm
will produce a model that generalizes well to new data.
The search space in ML often incorporates choices
associated with the hypothesis space and the features of the learning
algorithm itself (e.g., how optimization of the training loss is
performed).  Under this common perspective, both HO and ML essentially
boil down to \textit{nesting two search problems}: at the inner level
we seek a good hypothesis (as in standard supervised learning) while
at the outer level we seek a good configuration (including a good
hypothesis space) where the inner search takes place.  Surprisingly,
the literature on ML has little overlap with the literature on HO and
in this paper we present a unified framework encompassing both of them.

Classic approaches to HO ~\citep[see e.g.][for
a survey]{hutter_beyond_2015} have been only able to manage a relatively small number of
hyperparameters, from a few dozens using random
search~\cite{bergstra_random_2012} to a few hundreds using
Bayesian or model-based
approaches~\cite{bergstra_making_2013,snoek_practical_2012}.  
Recent gradient-based techniques for HO, however, have significantly
increased the number of hyperparameters that can be optimized
\cite{domke_generic_2012,maclaurin_gradient-based_2015,%
  pedregosa2016hyperparameter,franceschi_forward_2017}
and it is now possible to tune as hyperparameters entire weight
vectors associated with a neural network layer. In this way, it
becomes feasible to design models that possibly have more hyperparameters than parameters.
Such an approach is well suited for ML, since parameters are learned from a small dataset, whereas hyperparameters leverage multiple available datasets.

HO and ML only differ substantially in terms of the 
experimental settings in which they are evaluated. While in HO the available data is associated with a single task and split into a training set (used to tune the parameters) and a validation set (used to tune the hyperparameters), in ML we are often interested in
the so-called \textit{few-shot}
learning setting where data comes in
the form of short episodes (small datasets with few examples per
class) sampled from a common probability distribution over supervised
tasks.

Early work on ML dates
back at least to the 1990's~\cite{schmidhuber_learning_1992,baxter1995learning,thrun_learning_2012}
but this research area has received considerable attention in the
last few years, mainly driven by the need in real-life and industrial scenarios
for learning quickly a vast multitude of tasks. 
These tasks, or \emph{episodes}, may appear and evolve
continuously over time and may only contain few examples~\cite{lake_building_2017}.
Different strategies have emerged
to tackle ML. Although they do overlap in some aspects, it is
possible to identify at least four of them. The \textit{metric
  strategy} attempts to use training episodes to construct embeddings
such that examples of the same class are mapped into similar
representations. It has been instantiated in several variants that
involve non-parametric (or instance-based)
predictors~\cite{koch_siamese_2015,vinyals_matching_2016,%
  snell_prototypical_2017}.
In the related \textit{memorization strategy}, the meta-learner learns
to store and retrieve data points representations in
memory. It can be implemented either using recurrent
networks~\cite{santoro_meta-learning_2016} or temporal
convolutions~\cite{mishra2018ASimpleICLR}. The use of an attention
mechanism~\cite{vaswani_attention_2017} is crucial both
in~\cite{vinyals_matching_2016} and in~\cite{mishra2018ASimpleICLR}.  The
\textit{initialization strategy}~\cite{Sachin2017,%
  finn_model-agnostic_2017} uses training episodes to infer a good
initial value for the model's parameters so that new tasks can be
learned quickly by fine tuning. The
\textit{optimization
strategy}~\cite{andrychowicz_learning_2016,Sachin2017,wichrowska2017learnedICML} 
forges an optimization algorithm that will find it easier to learn on novel related tasks.

A
main contribution of this paper is
a unified view of HO and ML within
the natural mathematical framework of bilevel programming, where an outer
optimization problem is solved subject to the optimality of an inner
optimization problem. In HO the outer problem involves hyperparameters while
the inner problem is usually the minimization of an empirical loss.
In ML the outer problem could involve a shared representation among
tasks while the inner problem could concern classifiers for individual tasks. 
Bilevel programming~\cite{bard_practical_2013} has been suggested before in
machine learning in the context of kernel methods and support vector
machines~\citep{keerthi2007efficient,kunapuli_classification_2008},
multitask learning \cite{flamary2014learning},  
and more recently HO~\cite{pedregosa2016hyperparameter},
but never in the context of ML\@. The resulting framework outlined 
in Sec.~\ref{sec:framework} encompasses some existing approaches to
ML, in particular those based on the initialization and the optimization strategies. 

\begin{table}[t]
  \centering
  \caption{Links and naming conventions among different fields.}
  \label{tab:namings}
  \begin{small}
    \begin{tabular}{p{.27\textwidth}p{.27\textwidth}p{.3\textwidth}}
      \toprule
      {Bilevel\newline programming} &  {Hyperparameter\newline optimization} & {Meta-learning} \\
      \midrule
      {Inner variables} & {Parameters} & {Parameters of\newline Ground models} \\
      {Outer variables} & {Hyperparameters} & {Parameters of\newline Meta-learner} \\
      {Inner objective} & {Training error} & {Training errors\newline on tasks (Eq. 3)} \\
      {Outer objective} & {Validation error} & {Meta-training\newline error (Eq. 4)} \\
      \bottomrule
    \end{tabular}
  \end{small}
\end{table}

A technical difficulty arises when the solution to the inner problem
cannot be written analytically (for example this happens when using
the log-loss for training neural networks) and one needs to resort to
iterative optimization approaches. As a second contribution, we
provide in Sec.~\ref{sec:analysis} sufficient conditions that 
guarantee good approximation properties. We observe that these 
conditions are reasonable and apply to concrete problems relevant 
to applications.

In Sec. \ref{sec:hyper}, 
by taking inspiration on early work on representation learning in the context of multi-task and meta-learning \citep{baxter1995learning,caruana_multitask_1998},
we instantiate the framework 
for ML in a simple way treating the weights of the last layer of a neural 
network as the inner variables and the remaining weights,
which parametrize the representation mapping, 
as the outer variables. As shown in
Sec.~\ref{sec:ex}, the resulting ML algorithm performs well
in practice, outperforming most of the existing strategies 
on MiniImagenet. 

\section{A bilevel optimization framework}
\label{sec:framework}
In this paper, we consider bilevel optimization problems \citep[see
e.g.][]{colson2007overview} of the form 
\begin{equation}
\min \{ f(\lambda) : \lambda \in \Lambda\},
\label{eq:f}
\end{equation}
where function $f:\Lambda \rightarrow \mathbb{R}$ is defined at $\lambda\in \Lambda$ as
\begin{equation}
f(\lambda) = \inf \{ E(w_{\lambda}, \lambda ) : w_{\lambda} \in  {\rm arg}\min_{u\in \mathbb{R}^d} L_{\lambda}(u) \}.
\label{eq:def_f}
\end{equation}
We call $E:\mathbb{R}^d \times \Lambda\to\RSet$ the \emph{outer objective} 
and, for every $\lambda \in \Lambda$, we call $L_{\lambda}:\RSet^d\to\RSet$ the \emph{inner objective}.  
Note that $\{L_{\lambda} : \lambda \in \Lambda\}$ is
a class of objective functions parameterized by $\lambda$. 
Specific instances of this problem include HO and ML, which we discuss next.
Table \ref{tab:namings} outlines the links among bilevel programming,
HO and ML.


\subsection{Hyperparameter Optimization}
\label{sec:HO}

In the context of hyperparameter optimization, we are interested in
minimizing the validation error of a model
$g_w:\mathcal{X}\to\mathcal{Y}$ parameterized by a vector $w$, with
respect to a vector of hyperparameters $\lambda$.
For example, we may consider representation or regularization hyperparameters 
that control the hypothesis space or penalties, respectively.
In this setting, a prototypical choice for the inner objective is the
regularized empirical error 
$$
L_\lambda(w) = \sum_{(x,y) \in D_{\rm tr}}\ell(g_w(x),y) + \Omega_\lambda(w),
$$
where $D_{\operatorname{tr}}=\{(x_i,y_i)\}_{i=1}^n$ is a set of
input/output points,
$\ell$ is a prescribed loss function, and $\Omega_\lambda$ a regularizer
parameterized by $\lambda$. The outer objective represents a proxy for
the generalization error of $g_{w}$, and it may be given by the average loss on a
validation set $D_{\operatorname{val}}$
\[
E(w,\lambda) = \sum_{(x,y) \in D_{\rm val}} \ell(g_w(x),y).
\]
or, in more generality, by a cross-validation error, as detailed in Appendix B.
Note that in this setting, the outer objective $E$ does not depend explicitly 
on the hyperparameters $\lambda$, since in HO $\lambda$ is instrumental in finding 
a good model $g_w$, which is our final goal. As a more specific example, consider 
linear models, $g_w(x)=\langle w,x\rangle$, let $\ell$~be the square loss and let
$\Omega_\lambda(w) =\lambda \|w\|^2$, in which case the inner
objective is ridge regression 
(Tikhonov regularization) 
and the
bilevel problem optimizes over the regularization parameter the
validation error of ridge regression. 
\begin{figure}[t]
\includegraphics[width=.9\textwidth]{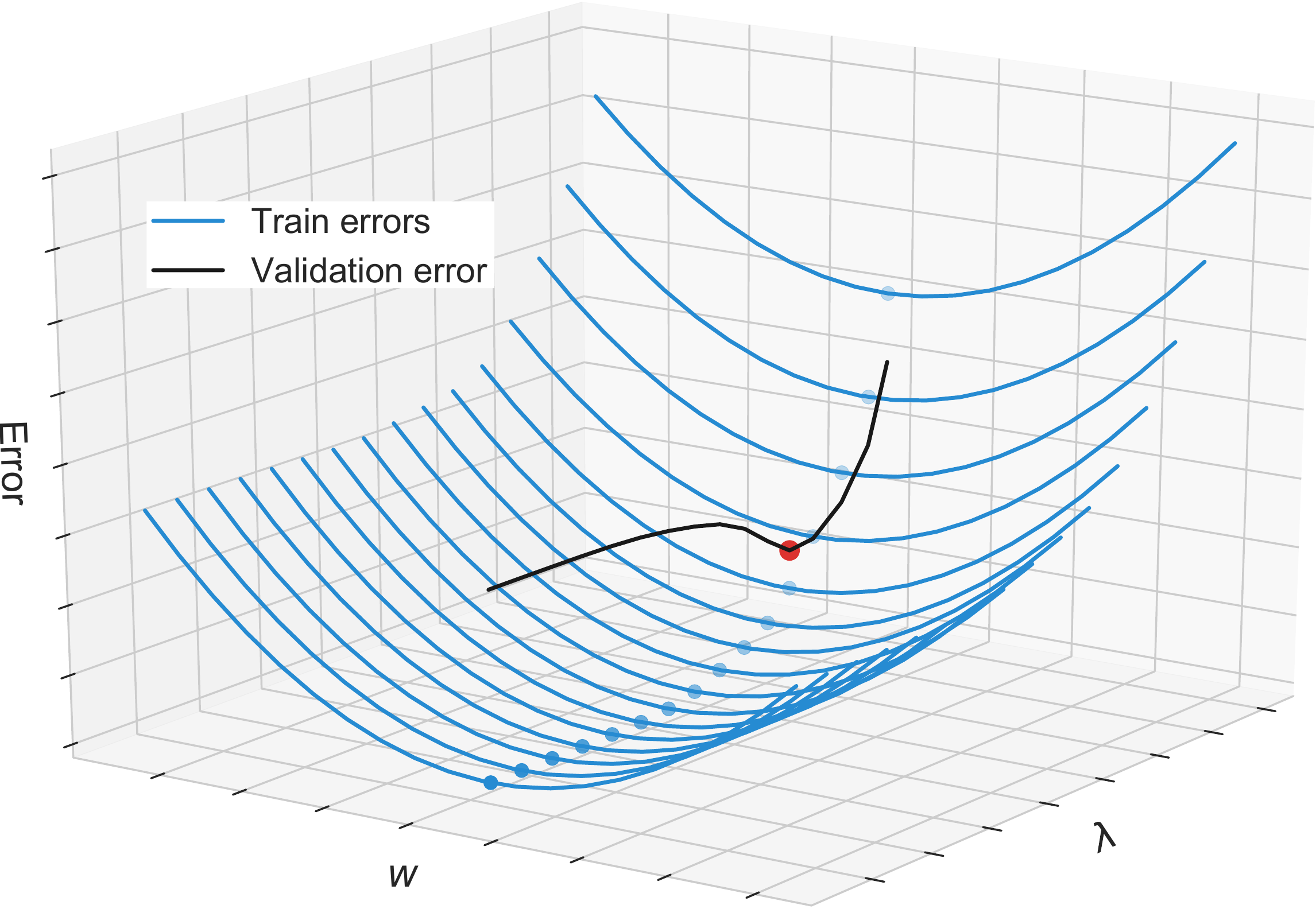}
\caption{\begin{small}
 \label{fig:bilevel} 
  Each blue line represents the average training error 
  when varying $w$ in $g_{w,\lambda}$ and the 
  corresponding inner minimizer is shown as a
  blue dot. The validation error evaluated at each minimizer yields the
  black curve representing the outer objective $f(\lambda)$, 
  whose minimizer is shown as a red dot. \end{small}}
\end{figure}

\subsection{Meta-Learning}
\label{sec:l2l}
In meta-learning (ML) the inner and outer
objectives are computed by averaging a training and a
validation error over multiple tasks, respectively. 
The goal is to produce a learning algorithm that will work well on novel 
tasks\footnote{The ML problem is also related to multitask learning, 
however in ML the goal is to extrapolate from the given tasks.}.
For this purpose, we have available a meta-training set
$\mathcal{D}=\{D^j=D^j_{\operatorname{tr}}\cup D^j_{\operatorname{val}}\}_{j=1}^{N}$, which is a
collection of datasets, 
sampled from a meta-distribution $\mathcal{P}$.
Each dataset
$D^j=\{(x_i^j,y_i^j)\}_{i=1}^{n_j}$ with
$(x_i^j,y_i^j)\in\mathcal{X}\times\mathcal{Y}^j$ is linked
to a specific task. Note that the output space is task dependent 
(e.g. a multi-class classification problem with variable number of classes). 
The model for each task is a function $g_{w^j,\lambda}:\mathcal{X}\to\mathcal{Y}^j$,
identified by a parameter vectors $w^j$ and hyperparameters $\lambda$. A key point
here is that $\lambda$ is shared between the tasks. With this notation the inner
and outer objectives are
\vspace{-5mm}
\begin{equation}
\label{eq:mlL}
L_\lambda(w) = \sum_{j=1}^N L^j(w^j, \lambda, D^j_{\operatorname{tr}}),
\end{equation}
\vspace{-3mm}
\begin{equation}
\label{eq:mlE}
E(w,\lambda) = \sum_{j=1}^N L^j(w^j, \lambda, D^j_{\operatorname{val}})
\end{equation}
respectively. The loss $L_j(w_j,\lambda,S)$ represents the empirical error of the pair $(w_j,\lambda)$ on a set of examples $S$. Note that the inner and outer losses for task $j$ use different train/validation splits of the corresponding dataset $D^j$. Furthermore, unlike in HO, in ML the final goal is to find a good $\lambda$ and the $w^j$ are now instrumental.

The cartoon in Figure \ref{fig:bilevel} illustrates ML as a
bilevel problem. 
The parameter $\lambda$ indexes an hypothesis space
within which the inner objective is minimized. A particular example, detailed in Sec. \ref{sec:hyper}, is to choose 
the model $g_{w,\lambda} = \langle w, h_\lambda(x)\rangle$, 
in which case $\lambda$ parameterizes a feature mapping. 
Yet another choice would be to consider
$g_{w^j,\lambda}(x) = \langle w + \lambda, x \rangle$, in which case $\lambda$ represents a common model around which task specific models are  to be found \citep[see e.g.][and reference therein]{evgeniou2005learning,finn_model-agnostic_2017,khosla2012undoing,kuzborskij2013n}.

\subsection{Gradient-Based Approach} \label{sec:gradapproach}
We now discuss a general approach to solve Problem (\ref{eq:f})-(\ref{eq:def_f}) when the hyperparameter vector $\lambda$ is real-valued.
To simplify our discussion let us assume 
that the inner
objective has a unique minimizer $w_\lambda$. Even in this simplified
scenario, Problem (\ref{eq:f})-(\ref{eq:def_f}) remains challenging to solve. Indeed, 
in general there is no closed form expression $w_\lambda$, so it is not possible to directly optimize the
outer objective function. 
While
a possible strategy (implicit differentiation) is to
apply the implicit function theorem 
to $\nabla L_{\lambda}=0$ 
\citep{pedregosa2016hyperparameter, 
koh2017understanding,
beirami2017optimal},
another 
compelling approach is to replace the 
inner problem with a dynamical system.
This point, discussed in
\cite{domke_generic_2012,maclaurin_gradient-based_2015,franceschi_forward_2017},
is developed further in this paper.

Specifically, we let $[T]=\{1,\dots,T\}$ where $T$ is a prescribed positive integer and consider the following approximation of Problem (\ref{eq:f})-(\ref{eq:def_f})
\begin{equation}
\label{eq:general:constrained}
\min\limits_{\lambda} f_T(\lambda) = E(w_{T,\lambda}, \lambda),
\end{equation}
where $E$ is a smooth scalar function, and\footnote{In general, the algorithm used to minimize the inner
  objective may involve auxiliary variables, e.g., velocities when
  using
  gradient descent with momentum,
  so $w$ should be intended as a larger vector containing both model
  parameters and auxiliary variables.}
\begin{equation}
\label{eq:general:constrained2}
w_{0,\lambda} = \Phi_0(\lambda),~w_{t,\lambda} =  \Phi_t(w_{t-1,\lambda},\lambda),~t \in [T],
\end{equation}
with $\Phi_0:\RSet^m\to\RSet^d$ a smooth initialization mapping and,
for every $t \in [T]$,
$\Phi_t : \RSet^d \times \RSet^m \rightarrow \RSet^d$ a smooth mapping
that represents the operation performed by the $t$-th step of an
optimization algorithm. For example, the optimization dynamics could
be gradient descent%
:
$\Phi_t(w_t, \lambda) = w_t - \eta_t \nabla L_{\lambda}(\cdot)$ where
$(\eta_t)_{t \in [T]}$ 
is
a 
sequence of steps sizes.

The approximation of the bilevel problem \eqref{eq:f}-\eqref{eq:def_f} by the procedure
\eqref{eq:general:constrained}-\eqref{eq:general:constrained2} raises the issue of the quality of this
approximation and we return to this issue in the next section. However, it also suggests to consider the inner dynamics as a form of approximate empirical error minimization 
(e.g. early stopping) 
which is valid in its own right. 
From this perspective
-- conversely to the implicit differentiation strategy --
it is possible
to include among the components
of $\lambda$ variables which are associated with the optimization
algorithm itself. For example, 
$\lambda$ may include the step sizes or momentum factors if the
dynamics $\Phi_t$ in Eq. \eqref{eq:general:constrained2} is
gradient descent with momentum; 
in \citep{andrychowicz_learning_2016,wichrowska2017learnedICML} the mapping $\Phi_t$ is implemented as a recurrent neural network, while \citep{finn_model-agnostic_2017} focus on the initialization mapping by letting $\Phi_0(\lambda) = \lambda$.   

A major advantage of 
this reformulation
is that it
makes it possible to compute efficiently  
the gradient of $f_T$, which we call \emph{hypergradient},
either in time or in memory
\citep{maclaurin_gradient-based_2015, franceschi_forward_2017}, by
making use of reverse or forward mode algorithmic differentiation 
\citep{griewank2008evaluating,baydin2017automatic}. 
This allows us
 to optimize a number of hyperparameters of the same order of
that of parameters, a situation which arise in
ML.

\section{Exact and Approximate Bilevel Programming} \label{sec:analysis}

In this section, we provide results 
about the existence of solutions of Problem~\eqref{eq:f}-\eqref{eq:def_f}
and the approximation properties of 
Procedure~\eqref{eq:general:constrained}-\eqref{eq:general:constrained2} 
with respect to the original bilevel problem%
. Proofs of these results are provided in the supplementary material.

Procedure~\eqref{eq:general:constrained}-\eqref{eq:general:constrained2},
though related to the bilevel problem \eqref{eq:f}-\eqref{eq:def_f},
may not be, in general, a good approximation of it. 
Indeed, making the assumptions  (which sound perfectly reasonable) that, for every $\lambda\in\Lambda$, 
$w_{T,\lambda}\to w_\lambda$ for some $w_\lambda\in\arg\max L_{\lambda}$, 
and that $E(\cdot, \lambda)$
is continuous, one can only assert that
$\lim_{T\to\infty} f_T (\lambda) = E(w_\lambda,\lambda) \geq f(\lambda)$. 
This is because the optimization dynamics converge to some minimizer of the inner objective $L_\lambda$, but not necessarily to the one that also minimizes the function $E$. This is illustrated in Figure \ref{fig:no_singleton}. 
The situation is, however, different if the inner problem admits a unique minimizer for every $\lambda\in\Lambda$. Indeed in this case, it is possible to show that the set of minimizers of the approximate problems converge, as $T \to +\infty$ and in an appropriate sense, to the set of minimizers of the bilevel problem.
More precisely, we make the following assumptions:
\begin{figure}[t] 
\includegraphics[width=.55\textwidth]{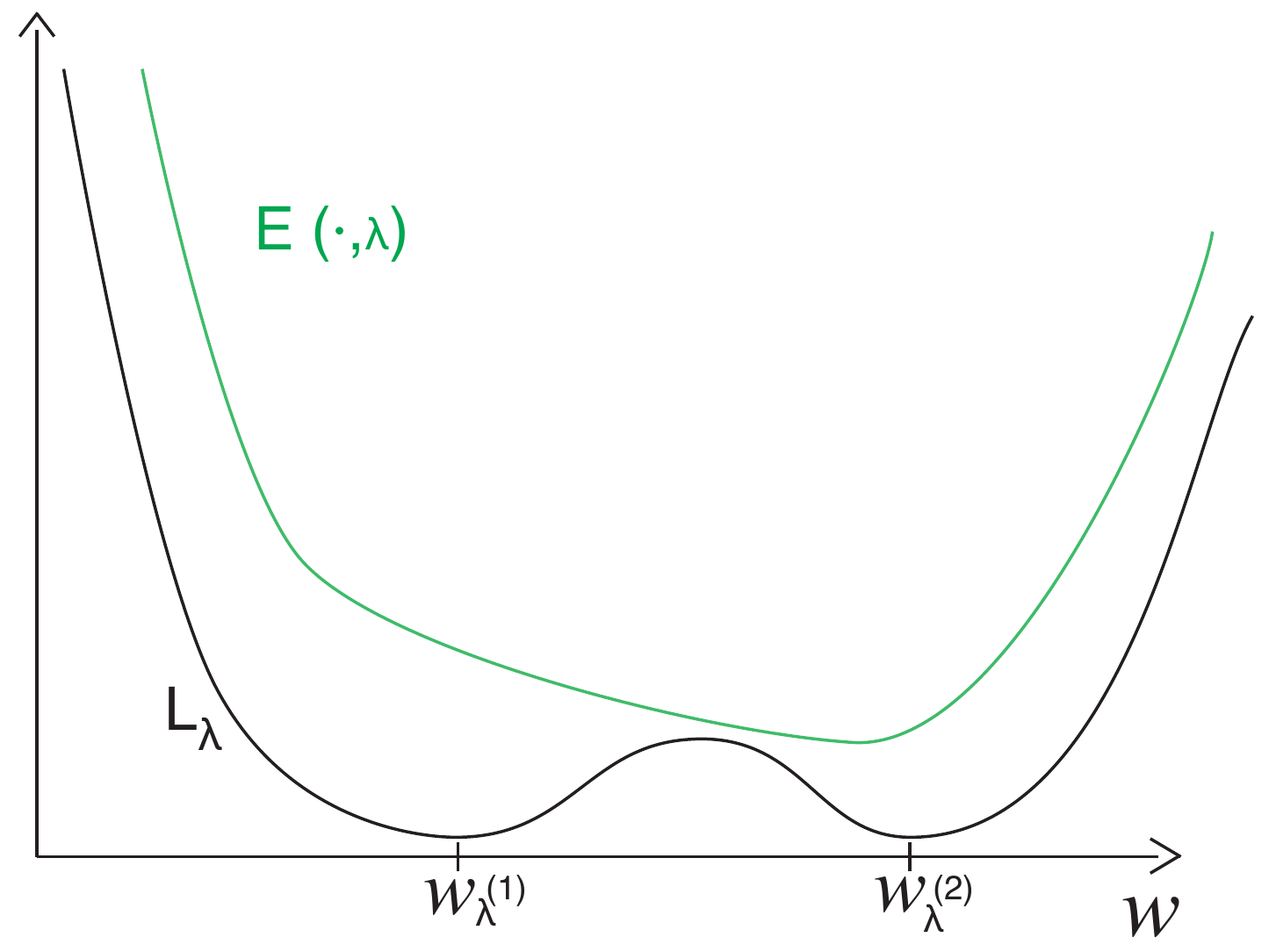}
\vspace{-5mm}
\caption{In this cartoon, for a fixed $\lambda$, $\mathrm{argmin} ~L_{\lambda}=\{w^{(1)}_{\lambda}, w^{(2)}_{\lambda} \}$; the iterates of an optimization mapping
$\Phi$ could converge to $w^{(1)}_{\lambda}$ with $E(w^{(1)}_{\lambda}, \lambda) > E(w^{(2)}_{\lambda}, \lambda)$.
} \label{fig:no_singleton}
\end{figure}
\vspace{-.20truecm}
\begin{enumerate}
\item[(i)] $\Lambda$ is a compact subset of $\RSet^m$;
\vspace{-.15truecm}
\item[(ii)] $E:\RSet^d\times\Lambda\to\RSet$ is jointly continuous;
\vspace{-.15truecm}
\item[(iii)] the map $(w,\lambda) \mapsto L_\lambda(w)$ is jointly continuous and such that $\arg\min L_{\lambda}$ is a singleton for every $\lambda\in\Lambda$;
\vspace{-.15truecm}
\item[(iv)] $w_\lambda=\arg\min L_{\lambda}$ remains bounded as $\lambda$ varies in $\Lambda$.
\end{enumerate}
\vspace{-.15truecm}
Then, problem~\eqref{eq:f}-\eqref{eq:def_f} becomes
\begin{equation}
\label{eq:mainprob}
\min_{\lambda \in \Lambda} f(\lambda) = E(w_\lambda, \lambda),
\qquad w_\lambda = \mathrm{argmin}_{u} L_\lambda(u).
\end{equation}
Under the above assumptions, in the following we give results about the existence of solutions of problem \eqref{eq:mainprob} and the (variational) convergence of the approximate problems
\eqref{eq:general:constrained}-\eqref{eq:general:constrained2}
 towards problem \eqref{eq:mainprob} --- relating the minima as well as the
 set of minimizers. In this respect we note that,
since both $f$ and $f_T$ are nonconvex, $\argmin f_T$ and $\argmin f$ 
are in general nonsingleton, so an appropriate definition of set 
convergence is required.

\begin{theorem}[Existence] \label{thm:existence}
Under Assumptions~{\rm (i)}-{\rm (iv)}
problem \eqref{eq:mainprob} admits solutions.
\end{theorem}
\vspace{-3mm}
\noindent{\textbf{Proof} See Appendix~A}.

The result below follows from 
general facts 
on the stability of minimizers in optimization problems \cite{dontchev93}. 

\begin{theorem}[Convergence]
\label{thm:main}
\label{p:20170207a}
In addition to Assumptions~{\rm (i)}-{\rm (iv)}, 
suppose that:
\vspace{-2ex}
\begin{enumerate}
\item[{\rm (v)}] $E(\cdot,\lambda)$ is uniformly Lipschitz continuous;
\vspace{-1ex}
\item[{\rm (vi)}] The iterates 
$(w_{T,\lambda})_{T \in \mathbb{N}}$ converge uniformly to $w_{\lambda}$ on $\Lambda$ as $T \to +\infty$.
\end{enumerate}
\vspace{-2ex}
Then 
\vspace{-.2truecm}
\begin{enumerate}
\item[{\rm (a)}] $\inf f_T \to \inf f$,
\vspace{-.1truecm}
\item[{\rm (b)}] $\argmin f_T \to \argmin f$, meaning that, for every 
$(\lambda_T)_{T \in \mathbb{N}}$ such that $\lambda_T \in \argmin f_T$, we have that:
\begin{itemize}
\item[-] $(\lambda_T)_{T \in \mathbb{N}}$ admits a convergent subsequence;
\item[-] for every subsequence $(\lambda_{K_T})_{T \in \mathbb{N}}$ such that $\lambda_{K_T} \to \bar{\lambda}$, we have $\bar{\lambda} \in \argmin f$.
\end{itemize}
\end{enumerate}
\end{theorem}
\vspace{-3mm}
\noindent{\textbf{Proof} See Appendix~A}.

We stress that assumptions (i)-(vi) are very natural and satisfied by many problems of practical interests.
Thus, the above results provide full theoretical justification to the proposed approximate procedure \eqref{eq:general:constrained}-\eqref{eq:general:constrained2}.
The following remark discusses assumption (vi), while
the subsequent example will be 
relevant to the experiments in Sec.~\ref{sec:ex}.

\begin{remark}
\label{rmk:050618a}
If $L_\lambda$ is strongly convex, then many gradient-based algorithms 
(e.g., standard and accelerated gradient descent) yield linear convergence of the iterates $w_{T,\lambda}$'s. Moreover, in such cases, the rate of linear convergence is of type $(\nu_\lambda - \mu_\lambda)/(\nu_\lambda + \mu_\lambda)$, where $\nu_\lambda$
and $\mu_\lambda$ are the Lipschitz constant of the gradient and the modulus of strong convexity of $L_\lambda$ respectively. 
So, this rate can be uniformly bounded from above by $\rho \in \left]0,1\right[$, provided that $\sup_{\lambda \in \Lambda}\nu_\lambda<+\infty$ and $\inf_{\lambda\in \Lambda}\mu_\lambda>0$.
Thus, in these cases $w_{T,\lambda}$ converges uniformly to $w_\lambda$
on $\Lambda$ (at a linear rate). 
\end{remark}

\begin{example}
\label{ex:linear}
Let us consider the following form of the inner objective:
\begin{equation}
L_H(w) = \lVert y - X H w\rVert^2 + \rho \lVert w \rVert^2, 
\end{equation}
where $\rho>0$ is a fixed regularization parameter and $H \in \R^{d \times d}$
is the hyperparameter, representing a linear feature map.  $L_H$
is strongly convex, with modulus $\mu = \rho>0$ (independent on the hyperparameter $H$),
and Lipschitz smooth with constant $\nu_H = 2 \lVert (X H)^\top X H + \rho I \rVert$,
which is bounded from above, if $H$ ranges in a bounded set of square matrices. 
In this case assumptions (i)-(vi) are satisfied.
\end{example}

\section{Learning Hyper-Representations}
\label{sec:hyper}

In this section, we instantiate the bilevel programming approach for
ML outlined in Sec. \ref{sec:l2l} in the case of deep learning where 
representation layers are shared across episodes. 
Finding good data representations is a centerpiece in machine learning.
Classical approaches~\cite{baxter1995learning,caruana_multitask_1998} 
learn both the weights of the representation mapping and those of the ground classifiers jointly on the same data. Here we follow the bilevel approach and split each dataset/episode in training and validation sets. 

\algsetup{indent=2em}
\begin{algorithm}[t]
  \caption{Reverse-HG for Hyper-representation}
  \label{alg:ho-reverse}
  \begin{algorithmic}
    \STATE {\bfseries Input:} $\lambda$, current values of the hyperparameter, $T$ number of iteration of GD, $\eta$ ground learning rate, $\mathcal{B}$ mini-batch of episodes from $\mathcal{D}$  
    \STATE {\bfseries Output:} Gradient of meta-training error w.r.t. $\lambda$ on $\mathcal{B}$
    \FOR {$j=1$ {\bfseries to} $|\mathcal{B}|$}
      \STATE $w^j_0 = 0$
      \FOR{$t=1$ {\bfseries to} $T$}
      \STATE $w^j_t\gets w_{t-1} - \eta \nabla_w L^j(w^j_{t-1}, \lambda, D^j_{\operatorname{tr}})$
      \ENDFOR
      \STATE $\alpha^j_T \gets \nabla_w L^j(w^j_T, \lambda, D_{\operatorname{val}})$
      \STATE $p^j \gets \nabla_{\lambda} L^j(w^j_T, \lambda, D_{\operatorname{val}})$
      \FOR {$t=T-1$ {\bfseries downto} $0$}
      \STATE       $p^j \gets p^j - \alpha^j_{t+1} \eta \nabla_{\lambda} \nabla_w L^j(w^j_{t}, \lambda, D^j_{\operatorname{tr}})$  %
      \STATE $\alpha^j_t \gets  \alpha^j_{t+1}
       \left[I - \eta \nabla_w \nabla_w L^j(w^j_{t}, \lambda, D^j_{\operatorname{tr}}) \right]$
      \ENDFOR
    \ENDFOR
    \STATE {\bf return} $\sum_j p^j$
  \end{algorithmic}
\end{algorithm}

Our method involves the learning of a cross-task intermediate
representation $h_\lambda:\mathcal{X}\to\mathbb{R}^k$ (parametrized by
a vector $\lambda$) on top of which task specific models
$g^j:\mathbb{R}^k\to\mathcal{Y}^j$ (parametrized by vectors $w^j$) are trained.
The final ground model for task $j$ is thus given by $g^j\circ h$.
To find $\lambda$, we solve Problem \eqref{eq:f}-\eqref{eq:def_f} 
with inner and outer objectives as in Eqs. 
\eqref{eq:mlL} and \eqref{eq:mlE}, respectively.
Since, in general, this problem cannot be solved exactly, we instantiate the approximation scheme in Eqs. \eqref{eq:general:constrained}-\eqref{eq:general:constrained2} as follows: 
\vspace{-2mm}
\begin{equation}
\label{eq:hr:outer}
\min\limits_{\lambda} f_T(\lambda) = \sum_{j=1}^N L^j(w^j_T, \lambda, D^j_{\operatorname{val}})
\end{equation}
\vspace{-2mm}
\begin{equation}
\label{eq:hr:dyn}
\hspace{-.2truecm}w^j_{t} =  w^j_{t-1} {-} \eta \nabla_w L^j(w^j_{t-1}, \lambda, D^j_{\operatorname{tr}})
,~t, j \in [T],\hspace{.1truecm}[N].
\end{equation}
Starting from an initial value, the weights of the task-specific models are learned by $T$ iterations of gradient descent.
The gradient of $f_T$ can be computed efficiently in time by making use of an extended 
reverse-hypergradient procedure~\citep{franceschi_forward_2017} which we present in 
Algorithm~\ref{alg:ho-reverse}. Since, in general, the number of episodes in a 
meta-training set is large, we compute a stochastic approximation of the gradient of $f_T$ by sampling a mini-batch of episodes.  
At test time, given a new episode $\bar{D}$, the representation $h$ is kept fixed,
and all the examples in  $\bar{D}$ are used to tune the weights $\bar{w}$ of the episode-specific model $\bar{g}$.

%
Like other initialization and optimization strategies for ML,
our method does not require lookups in a support set
as the memorization and metric
strategies do~\cite{santoro_meta-learning_2016,%
  vinyals_matching_2016, mishra2018ASimpleICLR}.  
Unlike~\cite{andrychowicz_learning_2016,Sachin2017} we do not tune the optimization algorithm, which in our case is plain empirical loss minimization by gradient descent, and rather focus on the hypothesis space.
Unlike~\cite{finn_model-agnostic_2017}, that aims at maximizing sensitivity
of new task losses to the model parameters, we aim at maximizing the generalization 
to novel examples during training episodes, with respect to $\lambda$. 
Our assumptions about the structure of the model are slightly stronger
than in~\cite{finn_model-agnostic_2017} but still mild, namely that some
(hyper)parameters define the
representation and the remaining parameters define the classification function.
In \citep{munkhdalai2017metaICML} the meta-knowledge is distributed among fast and slow weights and an external memory; our approach is more direct, since the meta-knowledge is solely distilled by $\lambda$. A further advantage of our method is that, if the episode-specific models are linear
(e.g. logistic regressors) and each loss $L^j$ is strongly convex in $w$, the theoretical
guarantees of Theorem~\ref{thm:main} apply (see Remark~\ref{rmk:050618a}). 
These assumptions are satisfied in 
the experiments reported in the next section.

\section{Experiments} \label{sec:ex}

The aim of the following experiments is threefold. First, we investigate the impact of the number of iterations of the optimization dynamics on the quality of the solution on a simple multiclass classification problem. Second, we test our hyper-representation method in the context of few-shot learning on two benchmark datasets. Finally, we constrast the bilevel ML approach against classical approaches to learn shared representations 
%
\footnote{The code for reproducing the experiments, based on the package \textsc{Far-HO} (\url{https://bit.ly/far-ho}), is available at \url{https://bit.ly/hyper-repr}}.

\subsection{The Effect of $T$}  \label{sec:ex:eT}

Motivated by the theoretical findings of Sec.~\ref{sec:analysis}, 
we empirically investigate how solving the inner problem \textit{approximately}
(i.e. using small $T$) affects convergence, generalization performances,
and running time. We focus in particular on the linear feature map
described in Example~\ref{ex:linear}, which allows us to compare the approximated solution
against the closed-form analytical solution given by
$$w_H = [(XH)^T XH + \rho I]^{-1}(XH)^T Y.$$
In this setting, the bilevel problem 
reduces to a (non-convex) 
%
%
optimization problem in $H$. 

%
%
We use a subset of 100 classes extracted from Omniglot dataset~\citep{lake_building_2017} to 
construct a HO problem aimed at tuning $H$.
A training set $D_{\operatorname{tr}}$ and a validation set $D_{\operatorname{val}}$,
each consisting of three randomly drawn examples per class, were sampled to form the HO problem. 
A third set $D_{\operatorname{test}}$, consisting of fifteen examples per class, was used for testing. 
Instead of using raw images as input, we employ feature vectors $x\in\mathbb{R}^{256}$ 
computed by the convolutional network trained on one-shot five-ways ML setting
as described in Sec.~\ref{sec:ex:hr}.


%
\begin{figure}[t] 
\includegraphics[width=0.9\textwidth]{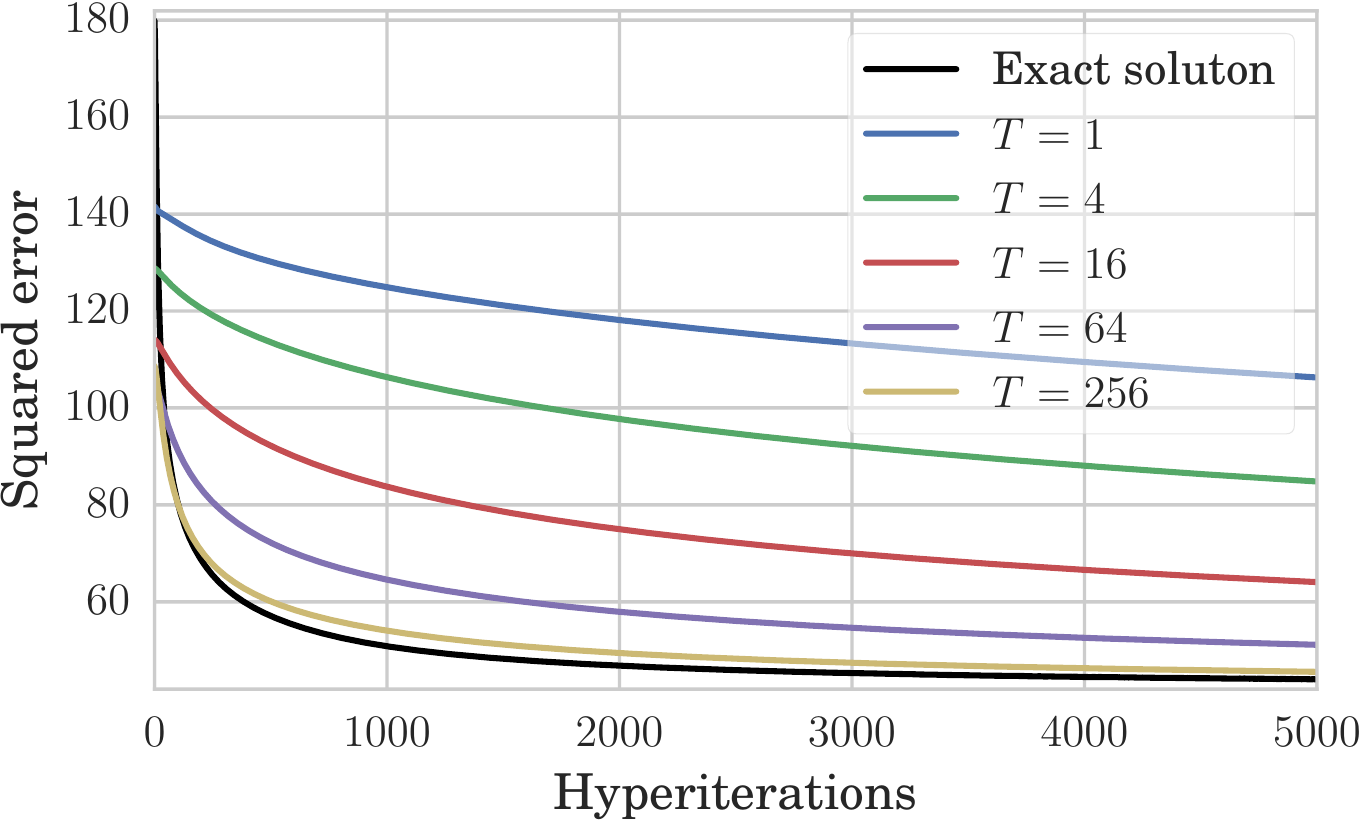}
\vspace{-5mm}
\caption{
Optimization of the outer objectives $f$ and $f_T$ for exact and approximate problems. The optimization of $H$ is performed with gradient descent with momentum, with same initialization, step size and momentum factor for each run.
}
\label{fig:eT:optimization}
\end{figure}
For the approximate problems we compute the hypergradient 
using Algorithm 1, where it is intended that $\mathcal{B}=\{(D_{\operatorname{tr}}, D_{\operatorname{val}})\}$. Figure~\ref{fig:eT:optimization} shows the values of functions $f$ and $f_T$ (see Eqs. \eqref{eq:f} and \eqref{eq:general:constrained}, respectively) 
during the optimization of $H$.
As $T$ increases, the solution of the approximate problem approaches the true bilevel solution.
%
%
%
However, performing a small number of gradient descent steps for solving the inner problem acts as implicit regularizer.
As it is evident from Figure \ref{fig:eT:generalization}, 
the generalization error is better when $T$ is smaller than the value yielding 
the best approximation of the inner solution.
%
\begin{figure}[t] 
\includegraphics[width=0.9\textwidth]{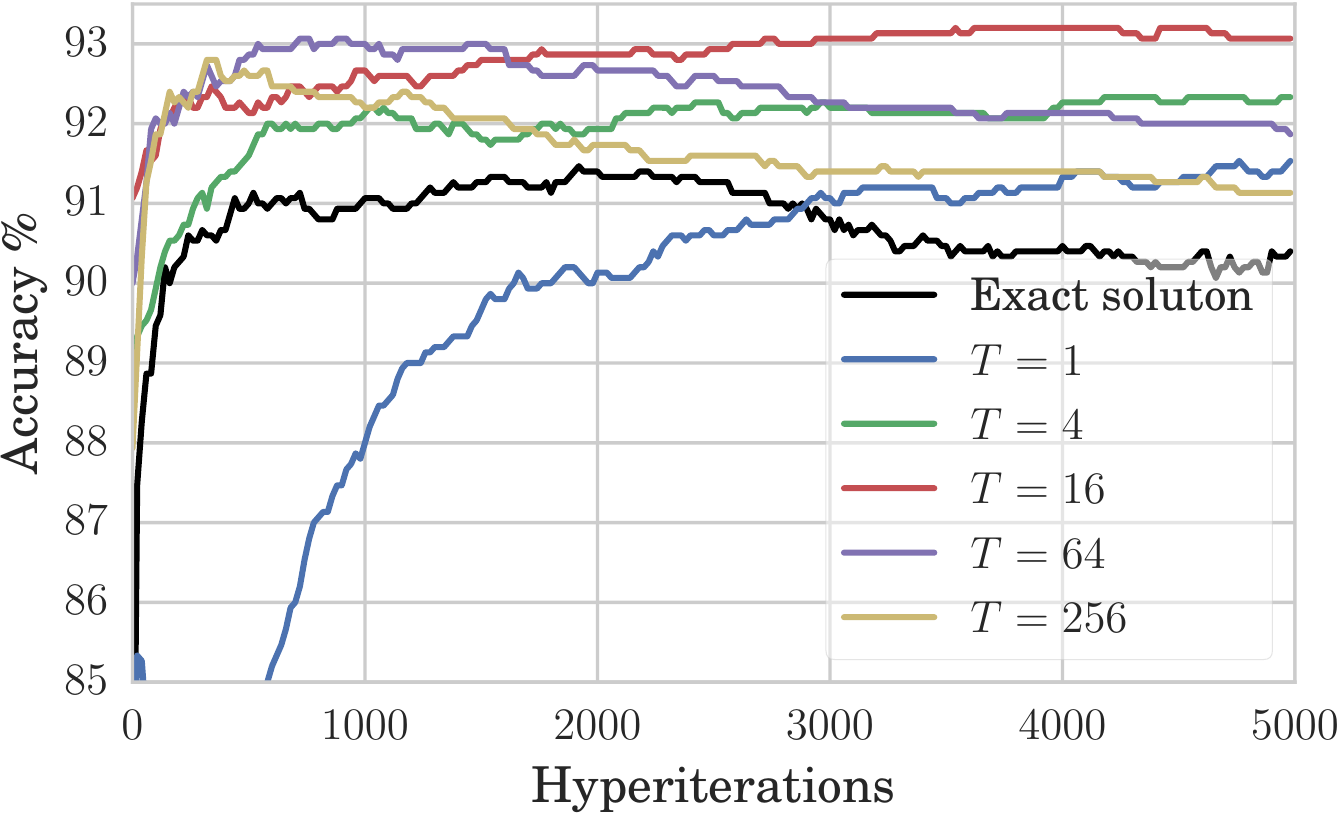}
\vspace{-5mm}
\caption{
Accuracy on $D_{\operatorname{test}}$ of exact and approximated solutions during optimization of $H$. Training and validation accuracies reach almost 100\% already for $T=4$ and after few hundred hyperiterations, and therefore are not reported.
}
\label{fig:eT:generalization}
\end{figure}
%
\begin{table}[t]
\vspace{-1mm}
\caption{Execution times on a NVidia Tesla M40 GPU.}
\label{tab:eT:exetime}
\begin{tabular}{r|ccccc|c}
	 $T$ & 1 & 4 & 16 & 64 & 256 & Exact \\
     \hline 
	 Time (sec) & 60 & 119 & 356 & 1344 & 5532 & 320 \\
\end{tabular}
\vspace{-6mm}
\end{table}
This is to be expected since, in this setting, the dimensions of parameters and hyperparameters are of the same order, leading to a concrete possibility of overfitting the outer objective (validation error).
An appropriate, problem dependent, choice of $T$  may help avoiding this issue (see also Appendix C).
As $T$ increases, the number of hyperiterations required to reach the maximum test accuracy decreases, 
further suggesting that there is an interplay between the number of iterations used to solve the inner and the outer objective. 
%
%
Finally, the running time of Algorithm \ref{alg:ho-reverse}, is 
linear in $T$ and the size of $w$ and independent of the size of $H$
(see also Table~\ref{tab:eT:exetime}), 
making it even more appealing to reduce the number of iterations. 

\subsection{Few-shot Learning}
\label{sec:ex:hr}

We new turn our attention to learning-to-learn, precisely to few-shot supervised learning, implementing the ML strategy outlined in Sec. \ref{sec:hyper} on two different benchmark datasets:
 
 \noindent $\bullet$ \textsc{Omniglot} \citep{lake2015human},
 a dataset that contains examples of 1623 different handwritten characters
 from 50 alphabets. We downsample the images to $28\times 28$. 
 
 \noindent $\bullet$ \textsc{MiniImagenet} \citep{vinyals_matching_2016},
a subset of ImageNet \citep{deng2009imagenet}, that contains 60000 downsampled images from 100 different classes.

Following the experimental protocol used in a number of recent works,  
%
%
we build a 
meta-training set
$\mathcal{D}$, from which we sample datasets to solve Problem \eqref{eq:hr:outer}-\eqref{eq:hr:dyn}, a meta-validation set 
$\mathcal{V}$ for tuning ML hyperparameters, and finally a meta-test set $\mathcal{T}$ which is used to estimate accuracy. Operationally, each meta-dataset consists of a pool of samples belonging to different 
(non-overlapping between separate meta-dataset)
classes, which can be combined to form ground classification datasets $D^j=D^j_{\operatorname{tr}} \cup D^j_{\operatorname{val}}$ with 5 or 20 classes (for Omniglot). 
The $D^j_{\operatorname{tr}}$'s contain 1 or 5 examples per class
which are used to fit $w^j$ (see Eq.~\ref{eq:hr:dyn}).
The $D^j_{\operatorname{val}}$'s, containing 15 examples per class, is used either
to compute $f_T(\lambda)$ (see Eq.~\eqref{eq:hr:outer}) and its (stochastic)
gradient if $D^j\in \mathcal{D}$ or to provide a generalization score if $D^j$
comes from either $\mathcal{V}$ or $\mathcal{T}$. 
For MiniImagenet we use the same split and images proposed in \citep{Sachin2017}, 
while for Omniglot we use the protocol defined by~\cite{santoro_meta-learning_2016}.

As ground classifiers we use multinomial logistic regressors and as task losses $\ell^j$ we employ cross-entropy.
The inner problems, being strongly convex, admit unique minimizers, yet require numerical computation of the solutions.
We initialize ground models parameters $w^j$ to $0$ and, according to the observation in Sec. \ref{sec:ex:eT}, we perform 
$T$ gradient descent steps, where $T$ is treated as a ML hyperparameter that has to be validated. 
Figure \ref{fig:metavalT} shows an example of meta-validation of $T$ for one-shot learning on MiniImagenet.
%
We compute a stochastic approximation of $\nabla f_T(\lambda)$ with Algorithm~\ref{alg:ho-reverse} 
and use Adam with decaying learning rate to optimize 
$\lambda$.

Regarding the specific implementation of the representation mapping $h$, we employ for Omniglot a four-layers convolutional neural network with strided convolutions and 64 filters per layer as in \citep{vinyals_matching_2016} and other successive works. For MiniImagenet we tried two different architectures:  

\noindent $\bullet$ \textit{C4L}, a four-layers convolutional neural network with max-pooling and 32 filters per layer; 

\noindent $\bullet$ \textit{RN}: a residual network \cite{he2016deep} 
built of four residual blocks followed by two convolutional layers.

The first network architecture has been proposed in \citep{Sachin2017} and then used 
in \citep{finn_model-agnostic_2017},
while a similar residual network architecture has been employed in a more recent
work~\citep{mishra2018ASimpleICLR}.
Further details on the architectures of $h$, as well as other ML hyperparameters, are specified in the supplementary material.
We report our results, using \emph{RN} for MiniImagenet, in Table~\ref{tab:results}, alongside scores from various recently proposed methods for comparison.

\begin{table*}[t]

\vspace{-2mm}
  \caption{Accuracy scores, computed on episodes from 
  $\mathcal{T}$, of various methods on 1-shot and 5-shot classification problems on Omniglot and MiniImagenet. For MiniImagenet 95\% confidence intervals are reported. 
  For Hyper-representation the scores are computed over 600 randomly drawn episodes. For other methods we show results as reported by their respective authors.\label{tab:results}}
\begin{small}
\begin{center}
\begin{tabular}{lcccccc}
      \toprule
  & \multicolumn{2}{c}{\textsc{Omniglot} 5 classes} & \multicolumn{2}{c}{\textsc{Omniglot} 20 classes}
  & \multicolumn{2}{c}{\textsc{MiniImagenet} 5 classes} \\
       Method & 1-shot & 5-shot & 1-shot & 5-shot & 1-shot & 5-shot 
      \\ \midrule
      \emph{Siamese nets} \citep{koch_siamese_2015}           & $97.3$ & $98.4$  & $88.2$ & $97.0$ & $-$ & $-$    \\
      \emph{Matching nets} \citep{vinyals_matching_2016} & $98.1$ & $98.9$  & $93.8$ & $98.5$  & $43.44 \pm 0.77$ & $55.31 \pm 0.73$ \\
       \emph{Neural stat.} \citep{edwards_towards_2016}  & $98.1$ & $99.5$  & $93.2$ & $98.1$ & $-$ & $-$  \\
      \emph{Memory mod.} \citep{Kaiser2017LearningICLR} & $98.4$ & $99.6$  & $95.0$ & $98.6$ & $-$ & $-$ \\
      \emph{Meta-LSTM} \citep{Sachin2017}  & $-$ & $-$ & $-$ & $-$ &  $43.56 \pm 0.84$ & $60.60 \pm 0.71$ \\
      \emph{MAML} \citep{finn_model-agnostic_2017} & $98.7$ & $99.9$  & $95.8$ & $98.9$ & $48.70\pm 1.75$ & $63.11 \pm 0.92 $   \\
      \emph{Meta-networks}
      \citep{munkhdalai2017metaICML}
      & $98.9$ &  $-$  & $97.0$ &  $-$  & $49.21\pm 0.96$ & $-$\\
      \emph{Prototypical Net.} \citep{snell_prototypical_2017} & $98.8$ &  $99.7$  & $96.0$ &  $98.9$  & $49.42\pm 0.78$ & $68.20\pm 0.66$ \\
      \emph{SNAIL} \citep{mishra2018ASimpleICLR}  & $99.1$ & $99.8$   & $97.6$ & $99.4$ & $55.71\pm 0.99$ & $68.88\pm0.92$ \\
      \hline
      \textbf{\textit{Hyper-representation}} & 98.6 & 99.5 &  95.5 &  98.4 & $50.54 \pm 0.85 $ &  $64.53 \pm 0.68$\\ 
\bottomrule
    \end{tabular}
\end{center}
    \end{small}
\vspace{-6mm}
\end{table*}
%
%
%
%
%
The proposed method achieves competitive results
highlighting the relative importance of learning a task independent representation, on the top of which 
logistic classifiers trained with very few samples generalize well. 
Moreover, utilizing more expressive models such as residual network as representation mappings, 
is beneficial for our proposed strategy and, unlike other methods, does not result in overfitting of the outer objective, as reported in \citep{mishra2018ASimpleICLR}.
Indeed, compared to \emph{C4L}, \emph{RN} achieves a relative improvement of 6.5\% on one-shot and 4.2\% on five-shot.
 Figure \ref{fig:repr} provides a visual 
example  
of the goodness of the learned representation, showing that MiniImagenet examples
(the first from meta-training, the second from the meta-testing sets) from similar classes (different dog breeds) are mapped near each other by $h$ and, conversely, 
samples from dissimilar classes are mapped afar.
\begin{figure}[h] 
\includegraphics[width=0.98\textwidth,height=2.2cm]{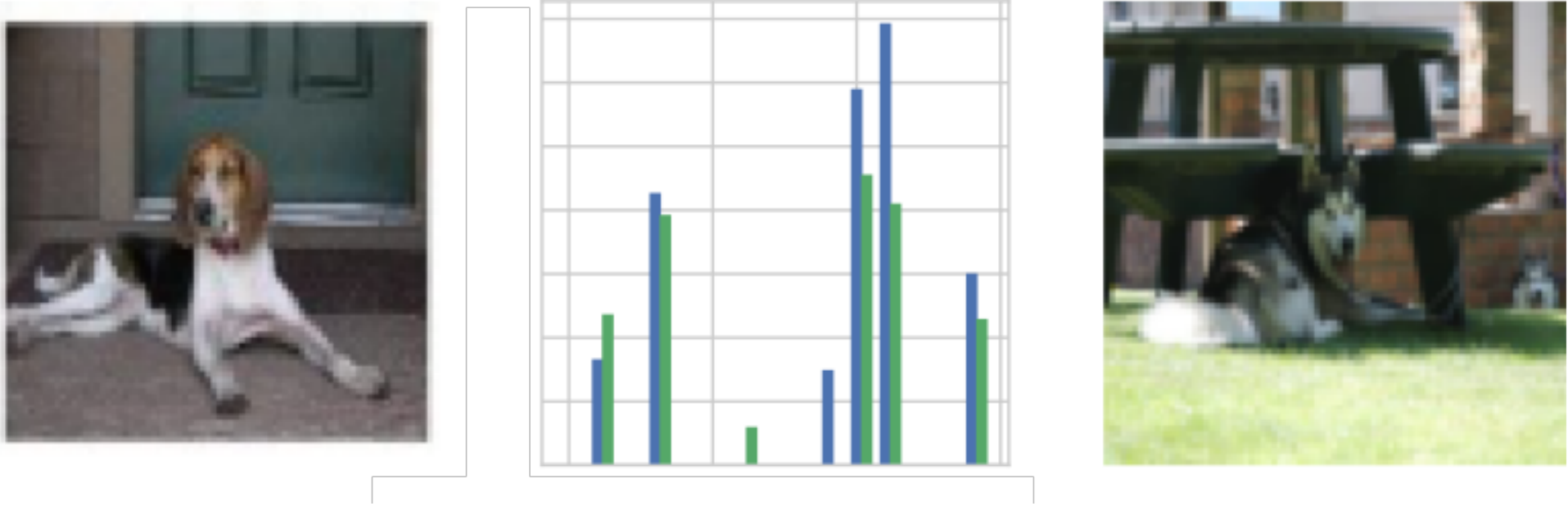}
\includegraphics[width=0.98\textwidth,height=2.2cm]{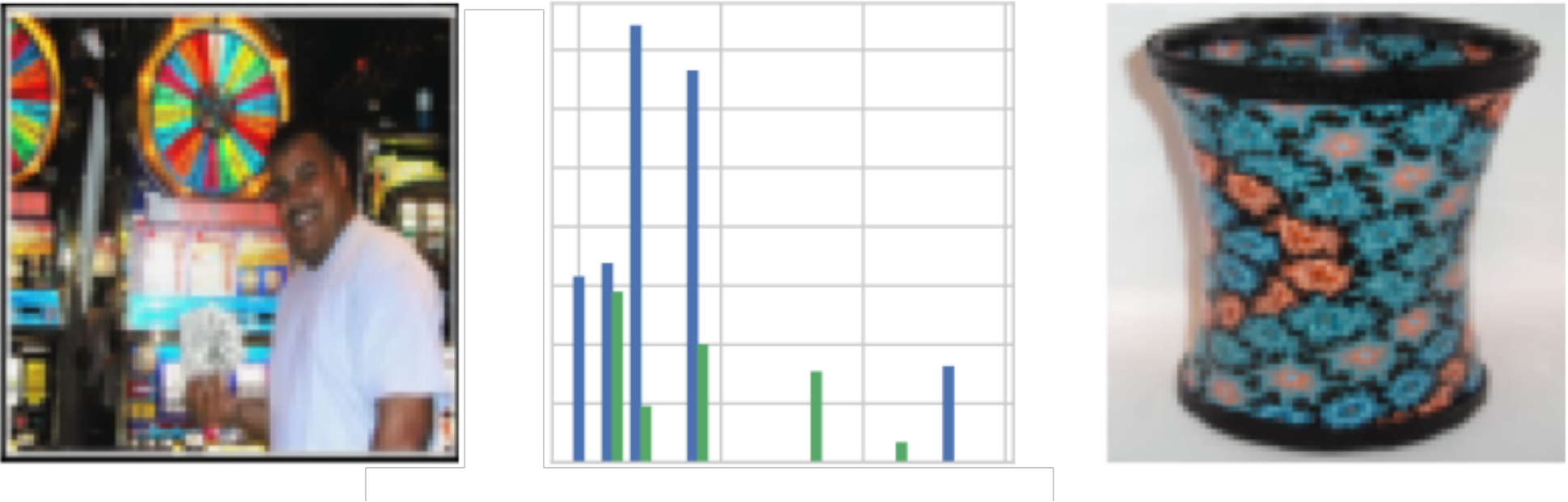}
\vspace{-5mm}
\caption{
\small{After sampling two datasets $D \in\mathcal{D}$ and $D'\in\mathcal{T}$, we show on the top the two images $x\in D,\; x'\in D'$  that minimize $||h_{\lambda}(x) - h_{\lambda}(x')||$ and on the bottom those that maximize it. In between each of the two couples we compare a random subset of components of $h_{\lambda(x)}$ (blue) and $h_{\lambda}(x')$ (green).} \label{fig:repr}}
\end{figure}
%
%
\begin{figure}[h] 
\vspace{-8mm}
\includegraphics[width=1.\textwidth]{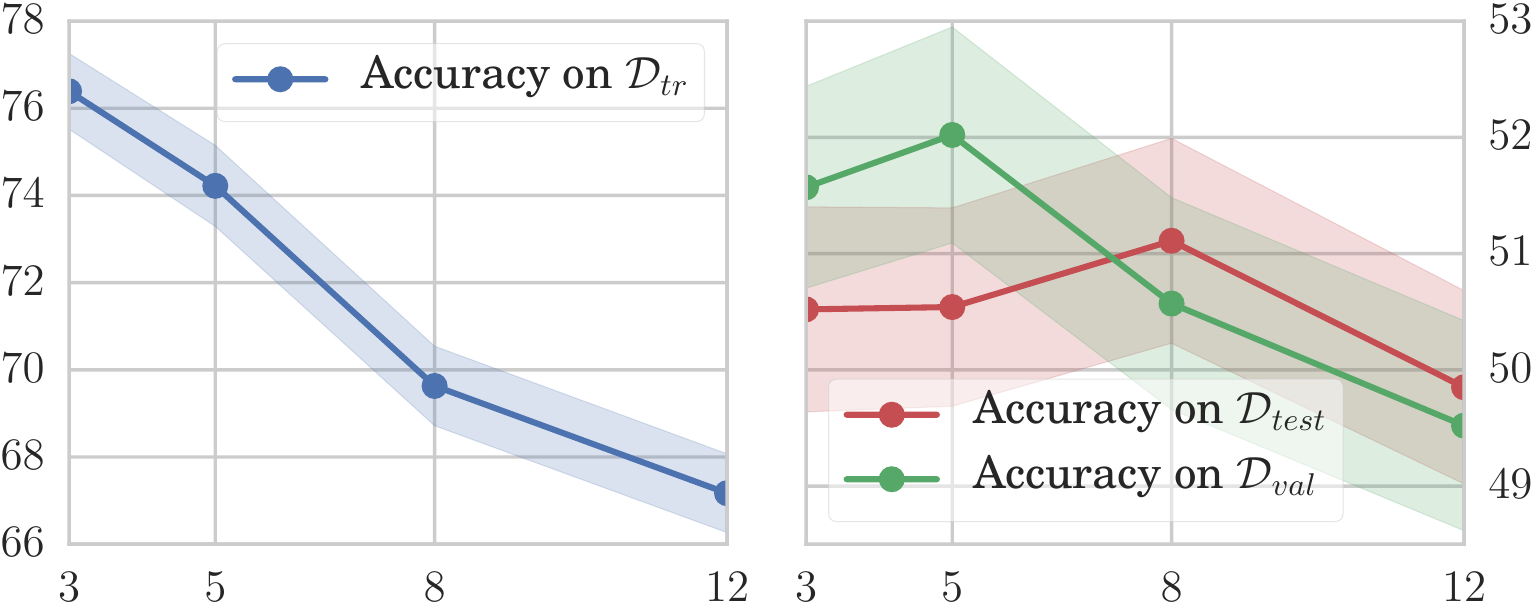}\caption{
Meta-validation of the number of gradient descent steps ($T$) of the ground models for MiniImagenet using the \emph{RN} representation. Early stopping on the accuracy on meta-validation set during meta-training resulted in halting the optimization of $\lambda$ after 
42k, 
40k, 
22k, 
and 15k
hyperiterations for $T$ equal to 3, 5, 8 and 12 respectively; in line with our observation in Sec. \ref{sec:ex:eT}.  
\label{fig:metavalT}
}
\vspace{-6mm}
\end{figure}
%

\vspace{-5mm}
\subsection{On Variants of Representation Learning Methods} \label{sec:ex:abrepr}

In this section, we show the benefits of learning a representation within the proposed bilevel framework 
compared to other possible approaches that involve
an explicit factorization of a classifier as $g^j\circ h$. The representation mapping $h$ is either pretrained 
or learned with different meta-learning algorithms. We focus on the problem of one-shot learning on MiniImagenet and we use \emph{C4L} as architecture for the representation mapping.
In all the experiments the ground models $g^j$ are multinomial logistic regressor as in Sec. \ref{sec:ex:hr}, tuned with 5 steps of gradient descent.
We ran the following experiments:

\noindent $\bullet$ \emph{Multiclass}:
the mapping $h:\mathcal{X}\to \mathbb{R}^{64}$ is given by the linear outputs before the softmax operation of a network\footnote{The network is similar to C4L but has 64 filters per layer.} 
pretrained on the totality of examples contained in the training meta-dataset
(600 examples for each of the 64 classes). In this setting, we found that using the second last layer or the output after the softmax yields worst results;

\noindent $\bullet$ \emph{Bilevel-train}: we use a bilevel approach but, unlike in Sec. \ref{sec:hyper}, we optimize the parameter vector $\lambda$ of the representation mapping by minimizing the loss on the training sets of each episode. The hypergradient is still computed with Algorithm \ref{alg:ho-reverse}, albeit we set $D^j_{\operatorname{val}}=D^j_{\operatorname{tr}}$ for each training episodes; 

\noindent $\bullet$ \emph{Approx} and \emph{Approx-train}: we consider an approximation of the hypergradient $\nabla f_T(\lambda)$ by  disregarding the optimization dynamics of the inner objectives (i.e. we set $\nabla_{\lambda} w^j_{T} = 0$). 
In \emph{Approx-train}  we just use the training sets;

\noindent $\bullet$ \emph{Classic}: as in \citep{baxter1995learning}, we learn $h$ by jointly optimize
$\hat{f}(\lambda, w^1, \dots, w^N) = 
    \sum_{j=1}^{N} L^j(w^j, \lambda, D^j_{\operatorname{tr}})$
 and treat the problem as standard multitask learning, with the exception that we evaluate $\hat{f}$ on mini-batches of 4 episodes, randomly sampled every 5 gradient descent iterations.

 In settings where we do not use the validation sets, we let the training sets of each episode contain 16 examples per class. Using training episodes with just one example per class resulted in performances just above random chance.
 While the first experiment constitutes a standard baseline, the others have the specific aim of assessing ($i$) the importance of splitting episodes of meta-training set into training and validation and ($ii$) the importance of computing the hypergradient of the approximate bilevel problem with Algorithm \ref{alg:ho-reverse}. 
The results reported in Table \ref{tab:variants} suggest that both the training/validation splitting and the full computation of the hypergradient constitute key factors for learning a good representation in a meta-learning context.
 \begin{table}[h]
 \vspace{-1mm}
\caption{
\small{Performance of various methods where the representation is either transfered or learned with variants of hyper-representation methods. The last raw reports, for comparison, the score obtained with hyper-representation}.} \label{tab:variants}
\begin{small}
  \begin{center}
    \begin{tabular}{lcc}
      \toprule
      Method & \# filters & Accuracy 1-shot \\
      \midrule
      \emph{Multiclass} &64 & $43.02$ \\ 
      \emph{Bilevel-train} &32 & $29.63$				\\
      \emph{Approx} 	& 	32 &	$41.12$		\\
      \emph{Approx-train} & 32& 		$38.80$			\\
      \emph{Classic-train} & 32 	&  $40.46$ 
      \\ \hline
      \emph{Hyper-representation-C4L} &32 & $47.51$ \\
      \bottomrule
    \end{tabular}
  \end{center}
\end{small}
\vspace{-9mm} 
\end{table}
 On the other side, using pretrained representations, especially in a low-dimensional space, turns out to be a rather effective baseline. One possible explanation is that, in this context, some classes in the training and testing meta-datasets are rather similar (e.g. various dog breeds) and thus ground classifiers can leverage on very specific representations.
 %

%
 

\vspace{-4mm}
\section{Conclusions}
\vspace{-1mm}

We have shown that both HO and ML can be formulated in terms of bilevel
programming and solved with an iterative approach.
When the inner problem has a unique solution (e.g. is strongly convex), our theoretical results show that
the iterative approach has convergence guarantees, 
a result that is interesting in its own right.
In the case of ML, by adapting classical strategies \citep{baxter1995learning} 
to the bilevel framework with training/validation splitting, we present a 
method for learning hyper-representations which 
is experimentally effective and supported by our theoretical guarantees.

Our framework encompasses recently proposed methods for meta-learning, such as learning to optimize, but also suggests 
different design patterns for the inner learning algorithm which could be interesting to explore in future work.
The resulting inner problems 
may
not satisfy the assumptions of our convergence analysis, raising the need for further theoretical
investigations. 
An additional 
future 
direction of research is the study of
the statistical properties of bilevel strategies where outer objectives
are based on the generalization ability of the inner model to new 
(validation) 
data. Ideas from \citep{maurer2016benefit,denvei2018incremental} may be useful in this direction.

\bibliographystyle{apalike}
\bibliography{all_refs}


\newpage


\appendix


\section{Proofs of the Results in Sec.~\ref{sec:analysis}}

{\em Proof of Theorem~\ref{thm:existence}.}
Since $\Lambda$ is compact, it follows from Weierstrass theorem
that a sufficient condition for the existence of minimizers is that $f$
is continuous. Thus,
let $\bar{\lambda} \in \Lambda$ and let $(\lambda_n)_{n \in \mathbb{N}}$
be a sequence in $\Lambda$ such that $\lambda_n \to \bar{\lambda}$. 
We prove that $f(\lambda_n) = E(w_{\lambda_n}, \lambda_n) \to E(w_{\bar{\lambda}},\bar{\lambda}) = f(\bar{\lambda})$.
Since $(w_{\lambda_n})_{n \in \mathbb{N}}$ is bounded, there exists
a subsequence $(w_{k_n})_{n \in \N}$ such that $w_{k_n} \to \bar{w}$
for some $\bar{w} \in \R^d$. Now, since $\lambda_{k_n} \to \bar{\lambda}$
and the map $(w,\lambda) \mapsto L_\lambda(w)$ is jointly continuous, we have
\begin{align*}
\forall\, w \in \R^d,\quad
L_{\bar{\lambda}}(\bar{w}) &= \lim_{n} L_{\lambda_{k_n}}(w_{k_n}) \\
&\leq \lim_{n} L_{\lambda_{k_n}}(w) = L_{\bar{\lambda}}(w).
\end{align*}
Therefore, $\bar{w}$ is a minimizer of $L_{\bar{\lambda}}$ and hence
$\bar{w} = w_{\bar{\lambda}}$. This prove that 
$(w_{\lambda n})_{n \in \N}$ is a bounded sequence
having a unique cluster point. Hence $(w_{\lambda n})_{n \in \N}$ is convergence
to its unique cluster point, which is $w_{\bar{\lambda}}$. Finally,
since $(w_{\lambda_n},\lambda_n) \to (w_{\bar{\lambda}}, \bar{\lambda})$
and $E$ is jointly continuous, we have 
$E(w_{\lambda_n}, \lambda_n) \to E(w_{\bar{\lambda}},\bar{\lambda})$
and the statement follows.
\qed

We recal a fundamental fact concerning 
the stability of minima and minimizers in optimization problems \cite{dontchev93}. We provide the proof for completeness.

\begin{theorem}[Convergence]
\label{thm:stability}
Let $\varphi_T$ and $\varphi$ 
be lower semicontinuous functions defined on 
a compact set $\Lambda$. Suppose that
$\varphi_T$ converges uniformly to $\varphi$ on $\Lambda$ as $T\to +\infty$.
Then 
\vspace{-.2truecm}
\begin{enumerate}
\item[{\rm (a)}] $\inf \varphi_T \to \inf \varphi$,
\vspace{-.1truecm}
\item[{\rm (b)}] $\argmin \varphi_T \to \argmin \varphi$, meaning that, for every 
$(\lambda_T)_{T \in \mathbb{N}}$ such that $\lambda_T \in \argmin \varphi_T$, we have that:
\begin{itemize}
\item[-] $(\lambda_T)_{T \in \mathbb{N}}$ admits a convergent subsequence;
\item[-] for every subsequence $(\lambda_{K_T})_{T \in \mathbb{N}}$ such that $\lambda_{K_T} \to \bar{\lambda}$, we have $\bar{\lambda} \in \argmin \varphi$.
\end{itemize}
\end{enumerate}
\end{theorem}
\begin{proof}
Let $(\lambda_T)_{T \in \N}$ be a sequence 
in $\Lambda$ such that,
for every $T \in \N$, $\lambda_T \in \argmin \varphi_T$. 
We prove that
\begin{enumerate}
\item[1)] $(\lambda_T)_{T \in \N}$ admits a convergent subsequence.
\item[2)] for every subsequence 
$(\lambda_{K_T})_{T \in \N}$ such that
$\lambda_{K_T} \to \bar{\lambda}$, we have $\bar{\lambda} \in \argmin \varphi$ and
$\varphi_{K_T}(\lambda_{K_T}) \to \inf \varphi$.
\item[3)] $\inf \varphi_T \to \inf \varphi$.
\end{enumerate}
The first point follows from the fact that $\Lambda$ is compact.\\
Concerning the second point, let $(\lambda_{K_T})_{T \in \N}$ be a subsequence
such that $\lambda_{K_T} \to \bar{\lambda}$. Since $\varphi_{K_T}$ converge uniformly
to $\varphi$, we have
\begin{equation*}
\lvert \varphi_{K_T}(\lambda_{K_T}) - \varphi(\lambda_{K_T}) \rvert 
\leq \sup_{\lambda \in \Lambda} \lvert \varphi_{K_T}(\lambda) - \varphi(\lambda) \rvert \to 0.
\end{equation*}
Therefore, using also the continuity of $\varphi$, we have
\begin{align*}
\forall\, \lambda \in \Lambda,\quad
\varphi(\bar{\lambda}) &= \lim_{T} \varphi(\lambda_{K_T}) 
=\lim_{T} \varphi_{K_T}(\lambda_{K_T}) \\
&\leq \lim_{T} \varphi_{K_T}(\lambda) = \varphi(\lambda).
\end{align*}
So, $\bar{\lambda} \in \argmin \varphi$ and $\varphi(\bar{\lambda}) = \lim_{T} \varphi_{K_T}(\lambda_{K_T}) \leq \inf \varphi = \varphi(\bar{\lambda})$, that is, $\lim_{T} \varphi_{K_T}(\lambda_{K_T}) = \inf \varphi$.\\
Finally, as regards the last point, we proceed by contradiction.
If $(\varphi_T(\lambda_T))_{T \in \N}$ does not convergce to $\inf f$,
then there exists an $\varepsilon>0$ and a subsequence 
$(\varphi_{K_T}(\lambda_{K_T}))_{T \in \N}$ such that
\begin{equation}
\label{eq:20180207c}
\lvert \varphi_{K_T}(\lambda_{K_T}) - \inf \varphi \rvert \geq \varepsilon, 
\quad \forall\, T \in \N
\end{equation}
Now, let $(\lambda_{K^{(1)}_T})$ be a convergent subsequence of 
$(\lambda_{K_T})_{T \in \N}$. Suppose that $\lambda_{K^{(1)}_T} \to \bar{\lambda}$.
Clearly $(\lambda_{K^{(1)}_T})$ is also a subsequence of $(\lambda_T)_{T \in \N}$.
Then, it follows from point 2) above that 
$\varphi_{K^{(1)}_T}(\lambda_{K^{(1)}_T}) \to \inf \varphi$. This latter finding together with
equation \eqref{eq:20180207c} gives a contradiction.
\end{proof}

{\em Proof of Theorem~\ref{thm:main}.}
Since $E(\cdot, \lambda)$ is uniformly Lipschitz continuous, there exists 
$\nu>0$ such that for every $T \in \mathbb{N}$
and every $\lambda \in \Lambda$
\begin{align*}
\lvert f_T(\lambda) - f(\lambda) \rvert 
&= \lvert E(w_{T,\lambda},\lambda) - E(w_\lambda,\lambda) \rvert\\
&\leq \nu \lVert w_{T,\lambda} - w_\lambda \rVert.
\end{align*}
It follows from assumption (vi)
that $f_T(\lambda)$ converges to $f(\lambda)$
uniformly on $\Lambda$ as $T\to +\infty$. Then
 the statement follows from Theorem~\ref{thm:stability}
\qed

\section{Cross-validation and Bilevel Programming} \label{sec:crossval}

We note that the (approximate) bilevel programming framework easily accommodates also estimations of the generalization error generated by a cross-validation procedures. We describe here the case of $K$-fold cross-validation, which includes also leave-one-out cross validation. 

Let $D=\{(x_i,y_i)\}_{i=1}^n$ be the set of available data; $K$-fold cross validation, with $K\in\{1, 2, \dots,  N\}$ consists in partitioning $D$ in $K$ subsets $\{D^j\}_{j=1}^K$ and fit as many models $g_{w^j}$ on training data $D^j_{\operatorname{tr}}=\bigcup_{i\neq j} D^i$. The models are then evaluated on $D^j_{\operatorname{val}}=D^j$. Denoting  by $w=(w^j)_{j=1}^K$ the vector of stacked weights, the $K$-fold cross validation error is given by
\begin{equation*}
E(w,\lambda) = \frac{1}{K}\sum_{j=1}^K E^j(w^j, \lambda)
\end{equation*}
where $E^j(w^j, \lambda)= \sum_{(x,y) \in D^j} \ell(g_{w^j}(x),y)
$.
$E$ can be treated as the outer objective in the bilevel framework, while the inner objective $L_{\lambda}$ may be given by the sum of regularized empirical errors over each $D^j_{\operatorname{tr}}$ for the $K$ models. 
Under this perspective, a $K$-fold cross-validation procedure closely resemble the bilevel problem for ML formulated in Sec. \ref{sec:l2l}, where, in this case, the meta-distribution collapses on the data (ground) distribution and the episodes are sampled from the same dataset of points.

By following the procedure outlined in Sec. \ref{sec:gradapproach} we can approximate the minimization of $L_{\lambda}$ with $T$ steps of an optimization dynamics and compute the hypergradient of $f_T(\lambda)=\frac{1}{K}\sum_j E^j(w^j_T,\lambda)$ by training the $K$ models and proceed with either forward or reverse differentiation. The models may be fitted sequentially, in parallel or stochastically. 
Specifically, in this last case, one can repeatedly sample one fold at a time (or possibly a mini-batch of folds) and compute a stochastic hypergradient that can be used in a SGD procedure in order to minimize $f_T$. 
At last, we note that similar ideas for leave-one out cross-validation error are developed in \cite{beirami2017optimal}, where the hypergradient of an approximation of the outer objective is computed by means of the implicit function theorem.  

\section{The Effect of $T$: Ridge Regression} \label{sec:Trr}

In Sec. \ref{sec:ex:eT} we showed that increasing the number of iterations $T$ leads to a better optimization of the outer objective $f$ through the approximations $f_T$, which converge uniformly to $f$ by Proposition \ref{p:20170207a}. This, however, does not necessary results in better test scores due to possible overfitting of the training and validation errors, as we noted for the linear hyper-representation multiclass classification experiment in Sec. \ref{sec:ex:eT}. The optimal choice of $T$ appears to be, indeed, problem dependent: if, in the aforementioned experiment, a quite small $T$ led to the best test accuracy (after hyperparameter optimization), in this section we present a small scale linear regression experiment that benefits from an increasing number of inner iterations.

We generated 90 noisy synthetic data points with 30 features, of which only 5 were informative, and divided them equally to form training, validation and test sets. As outer objective we set the mean squared error $E(w)=||X_{\operatorname{val}}w - y_{\operatorname{val}}||^2$,
where $X_{\operatorname{val}}$ and $y_{\operatorname{val}}$ are the validation design matrix and targets respectively and $w\in\RSet^{30}$ is the weight vector of the model. We set as inner inner objective  
$$L_{\lambda}(w) = ||X_{\operatorname{tr}}w - y_{\operatorname{tr}}||^2 + \sum_{i=1}^{30} e^{\lambda_i} w_i^2$$ 
and optimized the $L^2$ vector of regularization coefficients $\lambda$ (equivalent to a diagonal Tikhonov regularization matrix).
The results, reported in Table \ref{tab:eT:linearregr}, show that in this scenario overfitting is not an issue  and 250 inner iterations yield the best test result. 
Note that, as in Sec. \ref{sec:ex:eT}, also this problem admits an analytical solution, from which it is possible to compute the hypergradient analytically and perform exact hyperparameter optimization, as reported in the last row of the Table \ref{tab:eT:linearregr}.

\begin{table}[h]
  \caption{Validation and test mean absolute percentage error (MAPE) for various values of $T$.}
  \label{tab:eT:linearregr}
  \begin{tabular}{lcc}
    \toprule
    $T$   &     Validation MAPE  &  Test MAPE\\
    \midrule
    $10$ &         $11.35$  &      $43.49$\\
    $50$  &          $1.28$  &       $5.22$\\
    $100$ &          $0.55$  &       $1.26$\\
    $250$  &         $0.47$  &       $0.50$\\
    \hline
    Exact &        $0.37$  &       $0.57$ \\
    \bottomrule
  \end{tabular}
\end{table}

\section{Further Details on Few-shot Experiments}
\label{ap:modelDetails}
This appendix contains implementation details of the representation mapping $h$ and the meta-learning hyperparameters used in the few-shot learning experiments of Sec.~\ref{sec:ex:hr}.

To optimize  the representation mapping, in all the experiments, we use Adam with learning rate set to $10^{-3}$ and a decay-rate of $10^{-5}$. We used the initialization strategy proposed in \citep{glorot2010understanding} for all the weights $\lambda$ of the representation.

For Omniglot experiments we used a meta-batch size of 32 episodes for five-way and of 16 episodes for 20-way. To train the episode-specific classifiers we set the learning rate to 0.1.

For one set of experiments with Mini-imagenet we used an hyper-representation (\textit{C4L}) consisting of 4 convolutional layers where each layer is composed by a convolution with 32 filters, a batch normalization followed by a ReLU activation and a 2x2 max-pooling. The classifiers were trained using mini-batches of 4 episodes for one-shot and 2 episodes for five-shot with learning rate set to 0.01.

The other set of experiments with Mini-imagenet employed a Residual Network (\textit{RN}) as the representation mapping, built of 4 residual blocks (with 64, 96, 128, 256 filters) and then the block that follows $\{1\times 1$ conv (2048 filters), avg pooling, $1 \times 1$ conv (512 filters) $\}$. Each residual block repeats  the following block 3 times $\{1 \times 1$ conv, batch normalization, leaky ReLU (leak 0.1)$\}$ before the residual connection. In this case the classifiers were optimized using mini-batches of 2 episodes for both one and five-shot with learning rate set to 0.04.

\end{document}